\documentclass{article}

\usepackage[round,comma]{natbib}
\usepackage{fullpage}

\usepackage[utf8]{inputenc} 
\usepackage[T1]{fontenc}    
\usepackage{hyperref}       
\usepackage{url}            
\usepackage{booktabs}       
\usepackage{amsfonts}       
\usepackage{nicefrac}       
\usepackage{microtype}      

\usepackage[disable, colorinlistoftodos, textwidth=26mm, shadow,color=blue!30!white]{todonotes}

\usepackage{wrapfig}

\usepackage{algorithm,algorithmic}
\usepackage{verbatim}

\usepackage{amsmath,amsthm,amssymb,amsfonts}

\usepackage{enumerate}
\usepackage{authblk}

\usepackage{subfig}
\usepackage{graphicx}

\newcommand{\todoy}[2][]{\todo[color=red!20!white,#1]{Yasin: #2}}

\newcommand{\todom}[2][]{\todo[color=green!20!white,#1]{Mohammad: #2}}

\usepackage{amsmath,abbrevs}
\usepackage{amssymb,amsthm}

\usepackage{algorithm}
\usepackage{algorithmic}

\usepackage{color}
\usepackage{graphicx}
\usepackage{epstopdf}
\usepackage{algorithm,algorithmic}
\usepackage{verbatim}

\usepackage{nicefrac}

\newif\ifcomm
\commtrue

\newif\iflong
\longtrue

\newtheorem{theorem}{Theorem}[section]
\newtheorem{lemma}[theorem]{Lemma}

\newcounter{assumption}
\renewcommand{\theassumption}{A\arabic{assumption}}

\newcommand{\norm}[1]{\left\Vert#1\right\Vert}
\newcommand{\abs}[1]{\left\vert#1\right\vert}

\newcommand{\R}{{\mathbb{R}}}                        

\newcommand{\Prob}[1]{{\mathbb P}\left(#1\right)}    
\newcommand{\E}{{\mathbb E}}                         

\newcommand{\one}[1]{\mathbf 1 \left\{#1\right\}}           

\newcommand{\beq}{\begin{equation}}
\newcommand{\eeq}{\end{equation}}

\ifcomm
   \newcommand\comm[1]{\textcolor{blue}{ #1}}
   \newcommand{\mtodo}[2]{\todo{{\bf #1}: #2}} 
   \def\here#1{{\bf $\langle\langle$#1$\rangle\rangle$}}

\else
   \newcommand\comm[1]{}
   \newcommand{\mtodo}[2]{}
   \def\here#1{}
\fi

\renewcommand{\phi}{\varphi}

\newcommand{\cX}{{\cal X}}

\newcommand{\cA}{{\cal A}}

\def\eq#1{(\ref{eq-#1})}
\def\be{\begin{equation}}
\def\ee{\end{equation}}

\def\g{\gamma}

\def\simplex{\Delta}

\newname\controllermixture{{\rm \textsc{MDP-policy-mixture}}}
\newname\stableset{{\rm \textsc{independent-set}}}

\allowdisplaybreaks

\title{Optimizing over a Restricted Policy Class in Markov Decision Processes}

\author[1]{Ershad Banijamali}
\author[2]{Yasin Abbasi-Yadkori}
\author[3]{Mohammad Ghavamzadeh}
\author[4]{Nikos Vlassis}
\affil[1]{University of Waterloo}
\affil[2]{Adobe Research}
\affil[3]{DeepMind}
\affil[4]{Netflix}

\begin{document}

\maketitle

\begin{abstract}
We address the problem of finding an optimal policy in a Markov decision process under a restricted policy class defined by the convex hull of a set of base policies. This problem is of great interest in applications in which a number of reasonably good (or safe) policies are already known and we are only interested in optimizing in their convex hull.
We show that this problem is NP-hard to solve exactly as well as to approximate to arbitrary accuracy. However, under a condition that is akin to the occupancy measures of the base policies having large overlap, we show that there exists an efficient algorithm that finds a policy that is almost as good as the best convex combination of the base policies. The running time of the proposed algorithm is linear in the number of states and polynomial in the number of base policies. In practice, we demonstrate an efficient implementation for large state problems. Compared to traditional policy gradient methods, the proposed approach has the advantage that, apart from the computation of occupancy measures of some base policies, the iterative method need not interact with the environment during the optimization process. This is especially important in complex systems where estimating the value of a policy can be a time consuming process. 
\end{abstract}

\section{Introduction}
\label{sec:intro}

In many control problems, a number of reasonable base policies are known. For example these policies might be provided by an expert. A natural question is whether a combination of these base policies can provide an improvement. This problem is especially important when the number of states is large and the exact computation of the optimal policy is not feasible. We want to design reinforcement learning algorithms that can take this extra prior information into account. 
One way to formulate the problem is to define a policy space that includes all mixtures of the base policies. A member policy of this class is determined by its mixture distribution: It samples from this mixture at each state, as opposed to sampling from this mixture at the initial state and running the sampled policy until the end. 

A popular method to optimize a parameterized policy space is policy gradient that typically employs a variant of the gradient descent method~\citep{Williams92SS,Sutton-McAllester-Singh-Mansour-2000,Konda00AA,Baxter01IP,Peters05NA,Bhatnagar09NA}. Although in some applications the quality of the solution is high, the policy gradient methods often converge to some local minima as the problem is highly non-convex. Further, computing a gradient estimate can be an expensive operation. For example, the finite difference method requires running a number of policies in each iteration, and estimating the value of a policy in a complicated system might require a long running time.  

In this paper, we show a number of results on the problem of policy optimization in a restricted class of mixture policies. First, we show that the optimization problem is NP-hard to solve exactly as well as to approximate. The hardness result is obtained by a reduction from the \stableset problem for graphs and an application of the Motzkin-Straus theorem for optimizing quadratic forms over the simplex~\citep{Motzkin65}. The hardness result is somewhat surprising, since the same problem is known to be easy (in the complexity class P) if the space of base policies includes all MDP policies (an exponentially large space!)~\citep{Papadimitriou87}. The critical difference is that in the unconstrained case an optimal MDP policy is known to be deterministic, in which case linear programming or policy iteration are known to run in polynomial time~\citep{Ye05}, whereas in the restricted case an optimal policy may need to randomize.

Although this hardness result is somewhat disappointing, we show that an approximately optimal solution can be found in a reasonable time when the occupancy measures of the base policies have large overlap. We obtain this result by formulating the problem in the \emph{dual space}. More specifically, instead of searching in the space of mixture policies, we construct a new search space that consists of linear combinations of the occupancy measures of the base policies. Each such linear combination is not an occupancy measure itself, but it defines a policy through a standard normalization. The objective function in the dual space is still highly non-convex, but we can exploit the convex relaxation proposed by \citet{ABM2014} to have an efficient algorithm with performance guarantees. \citet{ABM2014} study the linear programming approach to dynamic programming in the dual space (space of occupancy measures), and propose a penalty method that minimizes the sum of the linear objective, and a multiple of the constraint violations.

To demonstrate the idea, consider the problem of controlling the service rate of a queue where jobs arrive at a certain rate and the cost is the sum of the queue length and service rate chosen. Consider two policies, one that selects low and one that selects high service rates. The space of the mixtures of these two policies is rich and is likely to contain a policy with low total cost. We can generate a wide range of service rates as a convex combination of these two base policies. Now let us consider the dual space. The occupancy measure of the first policy is concentrated at large queue lengths, while the occupancy measure of the second one is concentrated at short queue. It can be shown that the linear combination of occupancy measures can generate a limited set of policies that are either similar to the first policy or the second one. In short, although the space of mixture policies is a rich space (and hence the optimization is NP-hard in that space), the space of dual policies can be more limited. More crucially, if the base policies have some similarities so that their occupancy measures overlap, then we can generate non-trivial policies in the dual space that can be competitive with the mixture policies in the primal space.In the experiments section, we show that this is indeed the case. 


Given that we use occupancy measures as basis of the linear subspace, our approach has several advantages compared to \citet{ABM2014}. First, the importance weighting distributions are no longer needed and this simplifies the design of the algorithm. Second, our bounds are tighter as constraint violation terms no longer appear in the error bounds. Finally, unlike the algorithm of \citet{ABM2014}, our algorithm does not require a \emph{backward simulator}. Thus it can be used when only a forward simulator is given. This also provides a mechanism to run the algorithm using only a single trajectory of state-action observations.  Our algorithm requires the occupancy measures of the base policies as input. We employ recent results of \citet{LL-2017} to efficiently compute these vectors in time linear in the size of the state space. When the size of the state space is very large, these occupancy measures can be estimated by roll-outs. 

Let us compare our algorithm with the traditional policy gradient in the space of mixtures of policies. Although the space of the mixtures of policies is rich and is likely to contain a policy with lower total cost than any policy produced from a mixture of occupancy measures, our approach has several advantages. First, the policy gradient method is more computationally demanding. The gradient descent method needs to perform several roll-outs in each round to estimate the gradient direction. In a complicated system, the mixing times can be large and thus we might need to run a policy for a very long time before we can reliably estimate its value. In contrast and as we will show, apart from the initial roll-outs to estimate occupancy measures of the base policies, the proposed method doesn't need to interact with environment when optimizing in the dual space. Further, our method enjoys stronger theoretical guarantees than the policy gradient method. 

\subsection{Notation}

Let $M_{i,:}$ and $M_{:,j}$ denote $i$th row and $j$th column of matrix $M$ respectively. We use $I$ to denote an identity matrix. We use $1_n$ and $0_n$ to denote $n$-dimensional vectors with all elements equal to one and zero, respectively. We use $0_{mm}$ to denote the all-zeros $m\times m$ matrix, and
$e_i$ to denote the unit $m$-vector (1 at position $i$ and 0 elsewhere).
We also use $\one{.}$ to denote the indicator function. We use $\wedge$ and $\vee$ to denote the minimum and the maximum, respectively. We use the notation $[v]_{+}=v\vee 0$ with $\vee$ denotes the maximum. Similarly, we use the notation $[v]_{-}=v\wedge 0$ with $\wedge$ denotes the minimum in the paper. For vectors $v$ and $w$, $v \le w$ means element-wise inequality, i.e. $v_i \le w_i$ for all $i$. We use $\Delta_K$ to denote the space of probability distributions defined on the set $K$. For positive integer $m$, we use $[m]$ to denote the set $\{1,2,\dots,m\}$.

\section{Preliminaries}
\label{sec:prelim}

In this paper, we study Reinforcement Learning (RL) problems in which the interaction between the agent and environment has been modeled as a discrete\footnote{For simplicity, we consider finite-state problems in this paper. The results can be extended to continuous-state problems.} discounted Markov decision process (MDP). A discrete MDP is a tuple $\langle\cX,\cA,c,P,\alpha,\gamma\rangle$, where $\cX$ and $\cA$ are the sets of $X$ states and $A$ actions, respectively; $c:\cX\times\cA\rightarrow [c_{\min},c_{\max}]$ is the cost function; $P:\cX\times\cA\rightarrow\Delta_\cX$ is the transition probability distribution that maps each state-action pair to a distribution over states $\Delta_\cX$; $\alpha\in\Delta_\cX$ is the initial state distribution; and $\gamma\in(0,1)$ is the discount factor. We are primarily interested in the case where the number of states is large. Note that since we consider discrete MDPs, all the MDP-related quantities can be written in vector and matrix form. 

We also need to specify the rule according to which the agent selects actions at each possible state. We assume that this rule does not depend explicitly on time. A stationary policy $\pi:\cX\rightarrow\Delta_\cA$ is a probability distribution over actions, conditioned on the current state. The MDP controlled by a policy $\pi$ induces a Markov chain with the transition probability $P_\pi$ and cost function $c_\pi$. We denote by $J_\pi$ the value function of policy $\pi$, i.e.,~the expected sum of discounted costs of following policy $\pi$, and by $\nu_\pi\in\Delta_\cX$ and $\mu_\pi\in\Delta_{\cX\times\cA}$ the state and state-action occupancy measures under policy $\pi$ and w.r.t.~the starting distribution $\alpha$, respectively, i.e.,
\begin{align*}
\nu_\pi(x) = (1-\gamma)\sum_{x'\in\cX}\alpha(x')\sum_{t=0}^\infty\gamma^t\mathbb{P}(X_t=x|X_0=x'),
\end{align*}
and $\mu_\pi(x,a) = \nu_\pi(x) \, \pi(a|x)$.
Note that when $x_0\sim\alpha$, we may write $\mathbb{P}(X_t|X_0)=\alpha^\top P_\pi^t$, and thus, $\nu_\pi^\top=(1-\gamma)\alpha^\top(I-\gamma P_\pi)^{-1}$, where $I$ is the identity matrix. Given a policy class $\Pi$, the goal is to find a policy $\pi \in \Pi$ that minimizes $J(\pi)=\alpha^\top J_\pi$. It is easy to show that $J(\pi)=\nu_\pi^\top c_\pi=\mu_\pi^\top c$. 

In this work we are interested in the policy optimization problem when the policy class $\Pi$ is defined by a mixture of $m$ base policies $\pi_1,\ldots,\pi_m$:
\begin{equation*}
\Pi=\big\{\pi_w:\pi_w=\sum_{i=1}^mw_i\pi_i,\;w\in\Delta_{[m]}\big\}.
\end{equation*}
We call the policy space $\Pi$ the primal space. Executing a policy $\pi_w\in\Pi$ amounts to sampling one of the $m$ base policies from the distribution $w\in\Delta_{[m]}$ at each time step, and act according to this policy. Finding the best policy in $\Pi$ requires solving the following optimization problem 
\begin{equation}
\label{eq:optimization-prob}
\min_{w\in\Delta_{[m]}}J(\pi_w) = \min_{w\in\Delta_{[m]}}\alpha^\top J_{\pi_w} \;.
\end{equation}
We denote by $w^*$ the solution to~\eqref{eq:optimization-prob} and use $\pi_*=\pi_{w^*}$. For a positive constant $S$, we define the dual space of $\Pi$ as the space of linear combinations of the state-action occupancies of the base policies, i.e.,
\begin{equation}
\label{eq:dual-space}
\Xi = \big\{\xi_\theta:\xi_\theta=\sum_{i=1}^m\theta_i\mu_{\pi_i},\;\theta\in\Theta\big\},
\end{equation}
where $\Theta=\big\{\theta\in\R^m:\sum_{i=1}^m\theta_i=1,\;\norm{\theta}_2\leq S\big\}$. Here, $S$ is the radius of the parameter space and restricts the size of the policy class. Note that given the definition of $\Theta$, each $\xi_\theta\in\Xi$ is not necessarily a state-action occupancy measure. However, each $\xi_\theta\in\Xi$ corresponds to a policy $\pi_\theta$, defined as 
\begin{equation}
\label{eq:dual-policy}
\pi_\theta(a|x)=\frac{[\xi_\theta(x,a)]_{+}}{\sum_{a'\in\cA}[\xi_\theta(x,a')]_{+}} \;.
\end{equation}
If $\xi_\theta(x,a)\le 0$ for all $a\in\cA$, we let $\pi_\theta(.|x)$ be the uniform distribution. We denote by $\mu_{\pi_\theta}$ the state-action occupancy of this policy. The policy optimization problem in the dual space is defined as
\begin{equation}
\label{eq:policy-optimization-dual}
\min_{\theta\in\Theta}\;J(\pi_\theta) = \min_{\theta\in\Theta}\;\alpha^\top J_{\pi_\theta} = \min_{\theta\in\Theta}\;\mu_{\pi_\theta}^\top c\;,
\end{equation}
where $\pi_\theta$ is computed from $\theta$ using Eq.~\ref{eq:dual-policy}. 

\section{Hardness Result}

In this section we show that the policy optimization problem in the primal space (Eq.~\ref{eq:optimization-prob}) is NP-hard. At a high level, the proof involves designing a special MDP and $m$ base policies such that solving Eq.~\ref{eq:optimization-prob} in polynomial time would imply P$=$NP. 

\begin{theorem}
Given a discounted MDP, a set of policies, and a target cost $r$, 
it is NP-hard to decide if there exists a mixture of these policies that has expected cost at most $r$.
\end{theorem}
\begin{proof}
We reduce from the \stableset problem. 
This problem asks, for a given (undirected and with no self-loops) graph ${\mathcal G}$
with vertex set $V$,  
and a positive integer $j \leq |V|$, whether ${\mathcal G}$ contains an independent set $V' \subseteq V$ 
having $|V'| \geq j$. This problem is NP-complete~\citep{Garey79}. 

Let $G$ be the $m \times m$ (symmetric, 0\:\!-1) adjacency matrix of an input graph ${\mathcal G}$, 
and let $A=I+G$, where $I$ is the identity matrix. 
The reduction constructs a deterministic MDP with $m+3$ states and $m+3$ actions, 
and $m$ deterministic policies $\pi_i$ that induce corresponding chains $P_{\pi_i}$, for $i=1, \ldots, m$, where each $(m+3) \times (m+3)$ matrix $P_{\pi_i}$ reads
\begin{equation}
  P_{\pi_i} = 	\begin{bmatrix} 
			0_{~}		& 	e_i^\top 	& 	0		& 	0 			\\
			0_m			& 	0_{mm}  	& 	A_{:,i}		& 	1_m - A_{:,i} 	\\   
			0_{~}		& 	0_m^\top 	& 	0		& 	1	 		\\   
			0_{~}		& 	0_m^\top  	& 	0		& 	1	 		\\   
			\end{bmatrix} .
\end{equation}
A mixture $\pi_w$ of the base policies, with weights $w$, induces the chain
\begin{equation}
 P_{\pi_w} = \sum_{i=1}^m w_i P_{\pi_i} = 
            \begin{bmatrix} 
			0_{~}		& 	w^\top		& 	0		& 		0 				\\
			0_m			& 	0_{mm}  	& 	A w		& 	1_m - A w 	\\   
			0_{~}		& 	0_m^\top 	& 	0		& 	1	 		\\   
			0_{~}		& 	0_m^\top  	& 	0		& 	1	 		\\   
			\end{bmatrix} .
\end{equation}
It is easy to see that, for each $k \geq 3$, the $k$th power of $P_{\pi_w}$ is equal to 
$P_{\pi_w}^k = [ \, 0_{(m+3)(m+2)}, \, 1_{m+3} \, ]$ (all zeros except for the last column that is all ones), while its square $P_{\pi_w}^2$ reveals the quadratic form $w^\top A w$ in position $(1,m+2)$. Hence, for initial state distribution
$\alpha = [1, 0, \ldots, 0]^\top$ (all mass on the first state), state-only-dependent cost vector $c = [0, \ldots, 0, 1, 0]^\top$ (all zeros except for 1 in the $(m+2)$'th state),
and discount factor $\gamma < 1$, the expected discounted cost of $\pi_w$ is
\begin{align}
\notag
  J(\pi_w) &= (1-\g) \alpha^\top ( I - \g P_{\pi_w} )^{-1} c \\
  \notag
          &= (1-\g) \alpha^\top ( I + \g P_{\pi_w} + \g^2 P_{\pi_w}^2 + \ldots ) c\\
          &= (1-\g) \g^2 \, w^\top A w \;.
  \label{eq-J}
\end{align}
For any graph ${\mathcal G}$ with $m$ vertices and adjacency matrix $G$, the following identity holds~\citep{Motzkin65}:
\be
  \frac{1}{\omega(\mathcal G)} = \min_{y \in \simplex_{[m]}} \, y^\top (I + G) y \;,
\label{eq-motzkin}
\ee
where $\omega(\mathcal G)$ is the size of the maximum independent set of ${\mathcal G}$.
Let the target cost be $r = \frac{(1-\g) \gamma^2 }{j}$, where $j$ is the target integer of the \stableset instance. Then the decision question $J(\pi_w) \leq r$ is equivalent to $w^\top (G + I) w \leq \frac{1}{j}$, where we used \eq{J}. Hence, it follows from~\eq{motzkin} that the existence of a vector $w$ that satisfies $J(\pi_w) \leq r$
would imply $\omega(\mathcal G) \geq j$ and therefore $|V'| \geq j$ for some independent set $V' \subseteq V$. This establishes that, deciding the MDP policy optimization problem in polynomial time would also decide the \stableset problem in polynomial time, implying P$=$NP.
\end{proof}

\paragraph{Remarks:}
\begin{enumerate}
\item The same technique can be used to show NP-hardness of the problem under an average cost criterion. This only requires changing the last two rows of the matrices $P_{\pi_i}$, by having the 1's in the first column instead of the last column. In that case, 
if $\nu_{\pi_w} = [ \, x, \, v^\top, \, y, \, z \, ]^\top$ is the stationary distribution of $P_{\pi_w} $, where ${v}$ is an $m$-vector and $x,y,z$ scalars, we can algebraically solve the eigensystem $\nu_{\pi_w}^\top = \nu_{\pi_w}^\top P_{\pi_w}$ (by elementary manipulations), to get $v = w$ and $y = w^\top A w$. Hence, for a cost vector $c = [0, \ldots, 0, 1, 0]^\top$, the average cost of the MDP under $w$ is $w^\top A w$, and by choosing target cost $r = \frac{1}{j}$ the Motzkin-Straus argument applies as above.

\item The reduction via \stableset automatically establishes hardness of $\varepsilon$-approximability of the problem for arbitrary $\varepsilon > 0$ \citep{Haastad99,Zuckerman06}.

\item Optimizing over a restricted policy class essentially converts the MDP to a POMDP, for which related complexity results are known \citep{Papadimitriou87,Mundhenk00,Vlassis12}. 

\end{enumerate}


\section{Reduction to Convex Optimization}

\begin{figure}
\begin{center}
\framebox{\parbox{12cm}{
\begin{algorithmic}
\STATE \textbf{Input: } base policies $\{\pi_i\}_{i=1}^m$,; dual parameter space $\Theta$; number of rounds $T$, learning rates $\{\eta_t\}_{t=1}^T$
\STATE Compute occupancy measures of the base policies, i.e.,~$\{\mu_{\pi_i}\}_{i=1}^m$
\STATE Initialize $\theta_1 = 0$
\FOR{$t:=1,2,\dots, T$}
\STATE Sample $i\in [m]$ and sample $(x_t,a_t)\sim \mu_{\pi_i}$ 
\STATE Compute subgradient estimate $g_t(\theta_t)$ $\;$ {\em (using Eq.~\ref{eq:subgradient})}
\STATE Update $\theta_{t+1} = \Pi_{\Theta} (\theta_t - \eta_t g_t)$ $\;\;$  {\em ($\Pi_{\Theta}$ is the Euclidean projection onto $\Theta$)}
\ENDFOR
\STATE Compute $\widehat \theta = \frac{1}{T}\sum_{t=1}^T \theta_t$
\STATE Return policy $\pi_{\widehat \theta}\qquad\qquad\qquad\qquad\qquad$ {\em (using Eq.~\ref{eq:dual-policy})}
\end{algorithmic}
}}
\end{center}
\caption{The Stochastic Subgradient Method for MDPs.}
\label{alg:SGD}
\end{figure}

In this section, we first propose an algorithm to solve the policy optimization problem in the dual space (Eq.~\ref{eq:policy-optimization-dual}) and then prove a bound on the performance of the policy returned by our algorithm compared to the solution of the policy optimization problem in the primal space (Eq.~\ref{eq:optimization-prob}). 

Figure~\ref{alg:SGD} contains the pseudocode of our proposed algorithm. The algorithm takes the $m$ base policies $\{\pi_i\}_{i=1}^m$ and the dual parameter space $\Theta$ as input. It first computes the state-action occupancy measures of the base policies, i.e.,~$\{\mu_{\pi_i}\}_{i=1}^m$. This is done using the recent results by~\citet{LL-2017} that show it is possible to compute the stationary distribution (occupancy measure) of a Markov chain in time {\em linear} to the size of the state space.\footnote{These results apply to the discounted case as well.} 
This guarantees that we can compute the state-action occupancy measures of the base policies efficiently, even when the size of the state space is large. Alternatively, when the size of the state space is very large, these occupancy measures can be estimated by roll-outs.

\citet{ABM2014} propose minimizing a convex surrogate function which, when $\xi_\theta$ is a linear combination of occupancy measures, reads 
%
\begin{equation*}
\mathcal{L}(\theta) = c^\top \xi_\theta + H \sum_{(x,a)\in\mathcal{X}\times\mathcal{A}}\abs{[\xi_{\theta}(x,a)]_{-}} \;.
\end{equation*}
Here, $H$ is a parameter that penalizes negative values in $\xi_\theta$.  
At each time step $t$, our algorithm first computes an estimate of the sub-gradient $\nabla\mathcal{L}(\theta_t)$ and then feeds it to the projected sub-gradient method to update the policy parameters $\theta_t$. In order to compute an estimate of the sub-gradient $\nabla\mathcal{L}(\theta_t)$, we first sample a state-action pair $(x_t,a_t)$ from $(1/m) \sum_{i=1}^m \mu_{\pi_i}$, and then compute the function 
\begin{equation}
\label{eq:subgradient}
g_t(\theta_t) = c^\top M - H \frac{m\;M(x_t,a_t)}{\sum_{i=1}^m \mu_{\pi_i}(x_t,a_t)} \one{\xi_{\theta_t} (x_t,a_t) < 0 } \;,
\end{equation}
where $M$ is a $XA\times m$ matrix, whose $j$'th column is $\mu_{\pi_j}$, and $M(x,a)$ is the row of $M$ corresponding to the state-action pair $(x,a)$. To sample $(x_t,a_t)$ from $(1/m) \sum_{i=1}^m \mu_{\pi_i}$, we first select a number $i\in[m]$ uniformly at random and then sample a state-action pair $(x_t,a_t)$ from the state-action occupancy measure of the $i$'th base policy, i.e.,~$(x_t,a_t)\sim\mu_{\pi_i}$. We now show that $g_t(\theta)$ in~\eqref{eq:subgradient} is an unbiased estimate of $\nabla\mathcal{L}(\theta)$:
\begin{align*}
\E\big[g_t(\theta)\big] &= c^\top M - H\;\E\left[ \frac{M(x_t,a_t)}{(1/m)\sum_{i=1}^m \mu_{\pi_i}(x_t,a_t)} \one{\xi_{\theta} (x_t,a_t) < 0 } \right] \\
&= c^\top M - H \sum_{(x,a)\in\mathcal{X}\times\mathcal{A}} \frac{M(x,a) \Prob{x_t=x,a_t=a}}{(1/m)\sum_{i=1}^m \mu_{\pi_i}(x,a)} \one{\xi_{\theta} (x,a) < 0 } \\
&= c^\top M - H \sum_{(x,a)\in\mathcal{X}\times\mathcal{A}}  M(x,a) \one{\xi_{\theta} (x,a) < 0 } \\
&= \nabla \mathcal{L}(\theta) \;.
\end{align*}
%
After $T$ rounds, we average the computed policy parameters $\{\theta_t\}_{t=1}^T$ and obtain the final solution $\widehat{\theta} = (1/T)\sum_{t=1}^T \theta_t$. The parameter vector $\widehat{\theta}$ defines an element $\xi_{\widehat{ \theta}}$ of the dual space $\Xi$ (Eq.~\ref{eq:dual-space}), which in turn defines a policy $\pi_{\widehat{\theta}}$ using Eq.~\ref{eq:dual-policy}. We denote by $\mu_{\pi_{\widehat{\theta}}}$ the state-action occupancy measure of this policy. Our main result in this section is that $\pi_{\widehat \theta}$ is near-optimal in the dual space $\Xi$. This result is a consequence of Theorem~1 of \citet{ABM2014} applied to the dual space $\Xi$.
\begin{theorem}
\label{thm:cvx-red}
Let $\delta>0$ be the probability of error. Let $T=O(\frac{S^2}{\eta^2} \log(1/\delta))$ be the number of rounds and $H=1/\eta$ be the constraints multiplier in the subgradient estimate \eqref{eq:subgradient}. Let $\widehat\theta$ be the output of the stochastic subgradient method after $T$ rounds and let the learning rate be $\eta_t = S/(G'\sqrt{T})$, where $G' = \sqrt{m} + H m$. Then with probability at least $1-\delta$,
\begin{align*}
J(\widehat \theta) &= \alpha^\top J_{\pi_{\widehat \theta}} \le \min_{\theta\in\Theta} \Bigg( \alpha^\top J_{\pi_\theta}+ \left( \frac{1}{\eta} + \frac{6}{1-\gamma} \right) \sum_{(x,a)} \abs{[\xi_{\theta} (x,a)]_{-}}+ \frac{6\sqrt{m}C S \eta}{1-\gamma} + O(\eta) \Bigg) \,, 
\end{align*}
where the constants hidden in the big-O notation are polynomials in $S$, $m$, and $C$. 
\end{theorem} 
Let $U(\theta)=\sum_{(x,a)} \abs{[\xi_{\theta} (x,a)]_{-}}$ be the constraint violation. If we make the optimal choice of $\eta=\sqrt{(1-\gamma)U(\theta)}$, we get that
\begin{align*}
J(\widehat \theta) &= \alpha^\top J_{\pi_{\widehat \theta}} \le \min_{\theta\in\Theta} \Bigg( \alpha^\top J_{\pi_\theta}+ O\left( \max\left( \frac{U(\theta)}{1-\gamma}, \sqrt{\frac{U(\theta)}{1-\gamma}} \right) \right) \Bigg) \;. 
\end{align*}
This result highlights a number of interesting features of our approach. First, unlike the algorithm of \citet{ABM2014}, our algorithm does not require a backward simulator. Second, error terms involving stationarity constraints do not appear in the performance bound. 

Our next result shows that $\pi_{\widehat \theta}$ is near-optimal as long as the occupancy measures of the base policies have large overlap. 

\begin{theorem}
\label{thm:main}
Let $U(\theta)=\sum_{(x,a)\in\mathcal{X}\times\mathcal{A}} \abs{[\xi_\theta(x,a)]_{-}}$. We have the following bound on the performance of the policy $\pi_{\widehat{\theta}}$ returned by the algorithm in Figure~\ref{alg:SGD}:
\begin{align*}
&\overbrace{\alpha^\top J_{\pi_{\widehat{\theta}}}}^{J(\pi_{\widehat{\theta}})} \le \min_{w\in\Delta_{[m]}}\overbrace{\alpha^\top J_{\pi_w}}^{J(\pi_w)} + \min_{\theta\in\Theta} \Bigg( \sum_{i=1}^m \theta_i (\mu_{\pi_i} - \mu_{\pi_*})^\top c+ O\left(  \frac{U(\theta)}{1-\gamma} \vee \sqrt{\frac{U(\theta)}{1-\gamma}} \right) \Bigg) \;.
\end{align*}
Further, the computational complexity of the algorithm is $O\big(poly(X)A\big)$.
\end{theorem}

Before proving Theorem~\ref{thm:main}, we need to prove the following lemma. 

\begin{lemma}
\label{lem:muimuj}
Let $\varepsilon = \max_{i,j\in [m]} \norm{\nu_{\pi_i} - \nu_{\pi_j}}_1$. Then, for any $i\in [m]$ and any policy $\pi_w$ in the primal space $\Pi$, we have
\begin{equation*}
\norm{\nu_{\pi_i} - \nu_{\pi_w}}_1 \le \frac{\epsilon (1+\gamma)}{1-\gamma} \;. 
\end{equation*}
\end{lemma}
\begin{proof}
For any $i,j\in [m]$, there exists a vector $v_{i,j}$ with $\norm{v_{i,j}}_1 \le \epsilon$ such that 
\begin{equation}
\label{eq:v-eq0}
\nu_{\pi_i}-\nu_{\pi_j}=v_{i,j}.
\end{equation}
We may rewrite Eq.~\ref{eq:v-eq0} as
\begin{equation*}
\alpha^\top (I - \gamma P_{\pi_i})^{-1} = \alpha^\top (I - \gamma P_{\pi_j})^{-1} + v_{i,j}^\top \;,
\end{equation*}
and further as 
\begin{equation}
\label{eq:v-eq1}
\alpha^\top (I - \gamma P_{\pi_i})^{-1} (I - \gamma P_{\pi_j}) = \alpha^\top + v_{i,j}^\top (I - \gamma P_{\pi_j}) \;.
\end{equation}
Now for a policy $\pi_w\in\Pi$, corresponding to the weight vector $w\in\Delta_{[m]}$, we may write
\begin{align}
\alpha^\top (I - \gamma P_{\pi_i})^{-1} \Big( I - \gamma \sum_{k=1}^m w_k P_{\pi_k} \Big) &=\sum_{k=1}^m w_k \alpha^\top (I - \gamma P_{\pi_i})^{-1}  (I - \gamma P_{\pi_k}) \nonumber \\
&\stackrel{\text{(a)}}{=} \alpha^\top \sum_{k=1}^m w_k + \sum_{k=1}^m w_k v_{i,k}^\top (I - \gamma P_{\pi_k}) \nonumber \\
\label{eq:v-eq2}
&\stackrel{\text{(b)}}{=} \alpha^\top  + \sum_{k=1}^m w_k v_{i,k}^\top (I - \gamma P_{\pi_k}) \;,
\end{align}
where {\bf (a)} is from Eq.~\ref{eq:v-eq1} and {\bf (b)} is from the fact that $w\in\Delta_{[m]}$, and thus, $\sum_{k=1}^mw_k=1$. From Eq.~\ref{eq:v-eq2}, we have
\begin{align}
\label{eq:v-eq3}
\alpha^\top &(I - \gamma P_{\pi_i})^{-1} = \alpha^\top \Big( I - \gamma \sum_{k=1}^m w_k P_{\pi_k} \Big)^{-1}+ \sum_{k=1}^m w_k v_{i,k}^\top (I - \gamma P_{\pi_k}) \Big( I - \gamma \sum_{l=1}^m w_l P_{\pi_l} \Big)^{-1}.
\end{align}
Since $P_{\pi_w}=\sum_{k=1}^mw_kP_{\pi_k}$, we may rewrite Eq.~\ref{eq:v-eq3} as
\begin{align*}
\nu_{\pi_i} - \nu_{\pi_w} &= \alpha^\top (I - \gamma P_{\pi_i})^{-1} - \alpha^\top (I - \gamma P_{\pi_w})^{-1}\\ 
&= \sum_{k=1}^m w_k v_{i,k}^\top (I - \gamma P_{\pi_k}) \Big( I - \gamma \sum_{l=1}^m w_l P_{\pi_l} \Big)^{-1}.
\end{align*}
Let
\begin{equation*}
\epsilon' = \max_{i\in [m]} \norm{\sum_{k=1}^m w_k v_{i,k}^\top (I - \gamma P_{\pi_k}) \Big( I - \gamma \sum_{l=1}^m w_l P_{\pi_l} \Big)^{-1}}_1\;,
\end{equation*}
$z_k = (I - \gamma P_{\pi_k})^\top v_{i,k}$, $\;Q=\sum_{l=1}^m w_l P_{\pi_l}$, and $M^{-1} = I - \gamma Q$. We may write 
\begin{align*}
\norm{\nu_{\pi_i} - \nu_{\pi_w}}_1 &\leq \epsilon'\\ 
&\le \norm{M^\top \sum_{k=1}^m w_k z_k}_1 \\
&\le \norm{M^\top}_1 \norm{\sum_{k=1}^m w_k z_k}_1 \\
&\le \underbrace{\norm{I + \gamma Q^\top + \gamma^2 Q^{2\top}+\dots}_1}_{\leq(1+\gamma)}\sum_{k=1}^m w_k\overbrace{\norm{(I - \gamma P_{\pi_k})^\top}_1}^{\leq 1/(1-\gamma)} \overbrace{\norm{v_{i,k}}_1}^{\leq\epsilon} \\
&\le \frac{\epsilon (1+\gamma)}{1-\gamma} \;.
\end{align*}
This concludes the proof. 
\end{proof}

Now we are ready to prove Theorem~\ref{thm:main}. 
\begin{proof}[Proof of Theorem~\ref{thm:main}]
From Lemma~2 of \citet{ABM2014}, \todoy{refer to the journal version for the discounted result} for any $\theta\in\Theta$, we have
\begin{equation*}
\norm{\xi_\theta - \mu_{\pi_\theta}}_1 \le \frac{3 U(\theta)}{1-\gamma} \;.
\end{equation*}
From Lemma~\ref{lem:muimuj}, we have $\norm{\nu_{\pi_i} - \nu_{\pi_*}}_1 \le \epsilon'$, for any $i\in [m]$ and $\pi^*=\pi_{w^*}$. Thus, for any $\theta\in\Theta$, we may write
\begin{align*}
\alpha^\top J_{\pi_\theta} &= \alpha^\top J_{\pi_*} + \mu_{\pi_\theta}^\top c - \mu_{\pi_*}^\top c \\
&= \alpha^\top J_{\pi_*} + \xi_\theta^\top c - \mu_{\pi_*}^\top c + \mu_{\pi_\theta}^\top c - \xi_\theta^\top c    \\
&\le \alpha^\top J_{\pi_*} + (\xi_\theta-\mu_{\pi_*})^\top c + \frac{3 U(\theta)}{1-\gamma} \\
&= \alpha^\top J_{\pi_*} + \sum_i \theta_i (\mu_{\pi_i}-\mu_{\pi_*})^\top c + \frac{3 U(\theta)}{1-\gamma} \;.
\end{align*}
Let $b_i = (\mu_i-\mu_{\pi_*})^\top c$. We get that
\[
\alpha^\top J_{\pi_\theta} \le \alpha^\top J_{\pi_*} + \theta^\top b + \frac{3 U(\theta)}{1-\gamma} \;.
\]
Thus, by Theorem~\ref{thm:cvx-red},
\begin{align*}
\alpha^\top J_{\pi_{\widehat \theta}} &\le \alpha^\top J_{\pi_*} + \min_{\theta\in\Theta} \Bigg( \theta^\top b+ O\left( \max\left( \frac{U(\theta)}{1-\gamma}, \sqrt{\frac{U(\theta)}{1-\gamma}} \right) \right) \Bigg) \;. 
\end{align*}

\end{proof}


Let $\lambda$ be such that for any $i,j\in [m]$ and any $(x,a)\in\cX\times\cA$,
\beq
\label{eq:large-overlap}
\frac{\lambda}{1+\lambda} \mu_{\pi_j}(x,a) \le \mu_{\pi_i}(x,a) \le \frac{1+\lambda}{\lambda} \mu_{\pi_j}(x,a) \;.
\eeq
A large value of $\lambda$ indicates large overlap among occupancy measures. Let $\varepsilon = \norm{\mu_{\pi_1} - \mu_{\pi_*}}_1$. Parameter $\lambda$ provides an $\ell^\infty$ error and $\varepsilon$ provides an $\ell^1$ error bound. We can obtain an easy but non-trivial bound under the condition \eqref{eq:large-overlap} as follows. Let $\pi_1$ be the policy with smallest value and $\pi_m$ be the policy with largest value. Choose $\theta_1 = 1+\lambda$, $\theta_m=-\lambda$, and all other elements are zero. Then
\begin{align*}
U(\theta) &= \sum_{(x,a)} \abs{[\xi_{\theta} (x,a)]_{-}} \\
&= \sum_{(x,a)} \abs{[(1+\lambda) \mu_{\pi_1}(x,a) - \lambda \mu_{\pi_m}(x,a)]_{-}} \;.
\end{align*}
For each term to be negative, we must have
\[
\mu_{\pi_1}(x,a) \le \frac{\lambda}{1+\lambda} \mu_{\pi_m}(x,a)\,,
\]
which contradicts our assumption. Thus $U(\theta)=0$. We also get that
\begin{align*}
\theta^\top b &= (1+\lambda) (\mu_{\pi_1} - \mu_{\pi_*})^\top c - \lambda (\mu_{\pi_m} - \mu_{\pi_*})^\top c\\
&= (\mu_{\pi_1} - \mu_{\pi_*})^\top c - \lambda (\alpha^\top J_{\pi_m} - \alpha^\top J_{\pi_1}) \;.
\end{align*}
Thus,
\beq
\label{eq:bound1}
\alpha^\top J_{\pi_{\widehat\theta}} \le \alpha^\top J_{\pi_*} + \varepsilon - \lambda \alpha^\top (J_{\pi_m} - J_{\pi_1}) \;.
\eeq
Let's compare the above result to a simple bound:
\beq
\label{eq:bound2}
\alpha^\top J_{\pi_1} \le \alpha^\top J_{\pi_*} + \varepsilon \;.
\eeq
The term $\lambda \alpha^\top (J_{\pi_m} - J_{\pi_1})$ in \eqref{eq:bound1} shows the amount of improvement compared to the bound in \eqref{eq:bound2}. This term is  positive given that $\pi_m$ has a larger value than $\pi_1$.

\section{Experiments}
In this section we compare the results of combining policies in the policy space ($\theta$) and occupancy measure space ($w$), where the base policies have overlapping occupancy measures. We consider two different queuing problems. The results show that solving the problem in $\theta$ space results in a better optimal policy.

\subsection{Queuing Problem: Single Queue}
Consider a queue  of length L.  We denote the state of the system at time $t$ by  $x_t$, which shows  the number of jobs in the queue.  This problem has been studied before in \cite{de2003linear}. Jobs arrive at the queue with rate $p$. Action at each time, $a_t$, is chosen from the finite set  $\{0.1625,0.325,0.4875,0.65 \}$,  which  shows the service rate or departure  probability.  The transition function of the system is then defined as:
\begin{equation}
      x_{t+1}=\left\{
                \begin{array}{ll}
                  x_t - 1   	\hspace{1cm} \text{with probability:} \hspace{.4cm}  a_t\\
                  x_t + 1 		\hspace{1cm} \text{with probability:} \hspace{.4cm}  p\\
                  x_t           \hspace{1.6cm} \text{otherwise}        
                \end{array}
              \right.
\end{equation}
The system goes  from state $0$ to $1$ with probability $p$ and stays  in $0$ with probability $1-p$. Also, transition from state $L$ to $L-1$ has  probability $a(L)$ and the system stays in $L$ with probability $1-a(L)$. The cost incurred by being in state $x$ and taking action $a$ is given by: $c(x,a) = x^2 + 2500a^2$.

Consider a queue with $L =99$ and two  base policies.  The two policies, independent of the states,  have the following distributions over the set of actions:
$\pi_1 =  [0 , 0, 0.50, 0.50]$ and $\pi_2 = [0 ,0.1, 0.45, 0.45]$. Consider two scenarios. 1) Mixing the original policies in the space of $w$. 2) Building a policy based on combining the corresponding occupancy measures (optimization in the space of $\theta$). We intentionally choose two similar policies with high overlap in occupancy measure to show the effectiveness of finding a good mixture in the space of $\theta$ in this regime, while the gain in combining the original policies is not promising. 

Figure \ref{fig:experiment1} shows the results of optimization in the two spaces. The costs associated with $\pi_1$ and $\pi_2$ are $J({\pi_1}) = 831.91$ and $J({\pi_2}) = 777.36$. Suppose $\pi_w = w \pi_1 + (1-w) \pi_2$ is our mixture policy. Then for $0 \leq w \leq 1$, $J(\pi_w)$ is a monotonically increasing function of $w$. In other words, there is no gain in mixing the policies when $w$ is between $0$ and $1$. For this experiment we let $w$ to take values outside this interval to examine the lowest possible cost. The best mixing weight is $w = -5.49$ and the associated cost is $J(\pi_w) = 533.60$. However, in the space of occupancy measures we can build a better policy. Suppose $\xi_{\theta} = \theta \mu_1 + (1-\theta) \mu_2$. Then the cost associated with the policy that is built based on $\xi_{\theta}$ decreases monotonically as we decrease $\theta$ and it saturates at almost $J(\pi_{\theta}) = 447$, 
\todo{Nikos: the x-label is not shown in the theta plot}
which is much lower than the cost of optimal policy mixing in the space of $w$.
\begin{figure}[!h]
    \centering
    \subfloat[]{{\includegraphics[trim = 10mm 50mm 6mm 50mm,height=4.5cm]{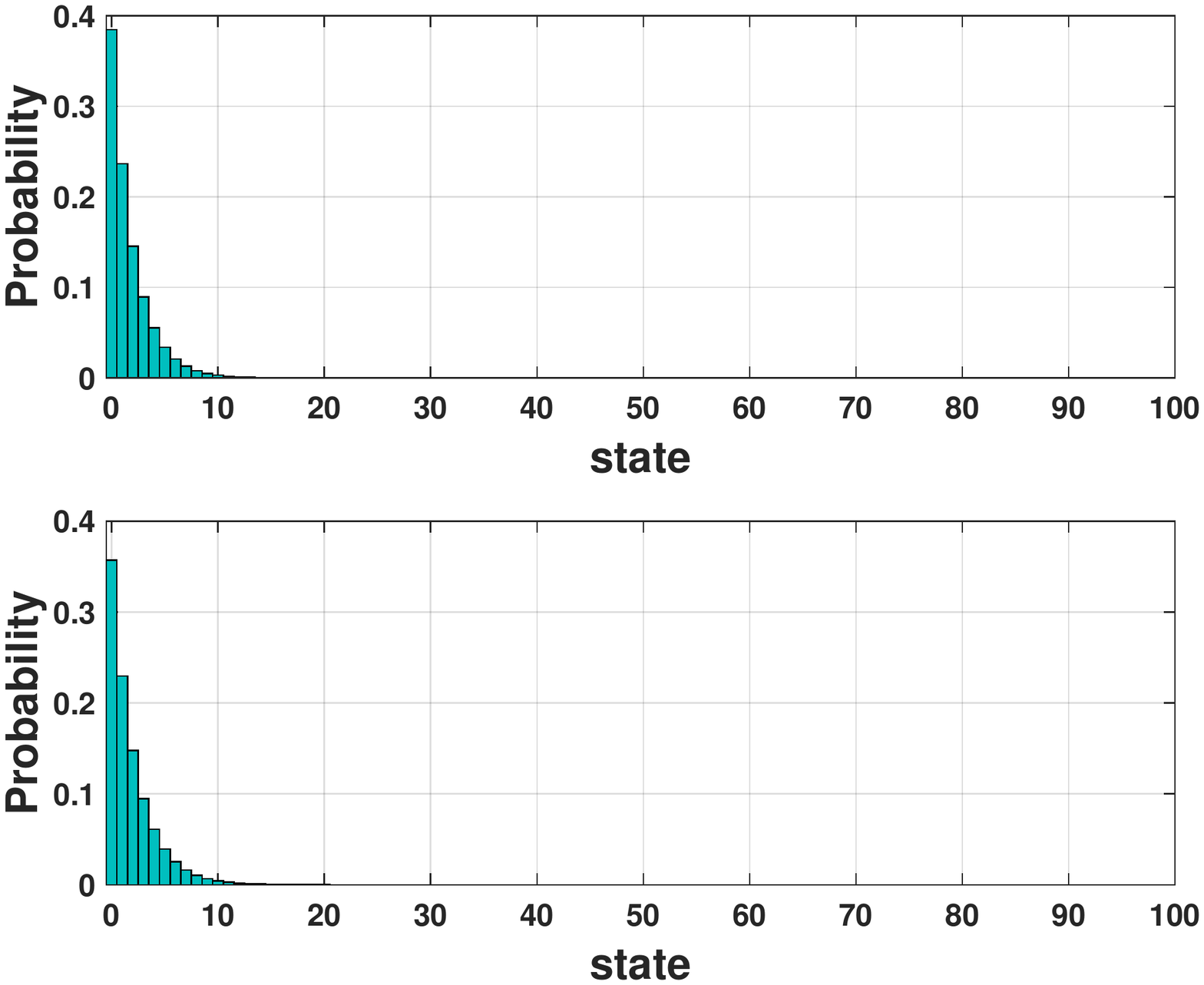} 	}}
	\subfloat[]{{\includegraphics[trim = 10mm 62mm 6mm 62mm,height=4.5cm]{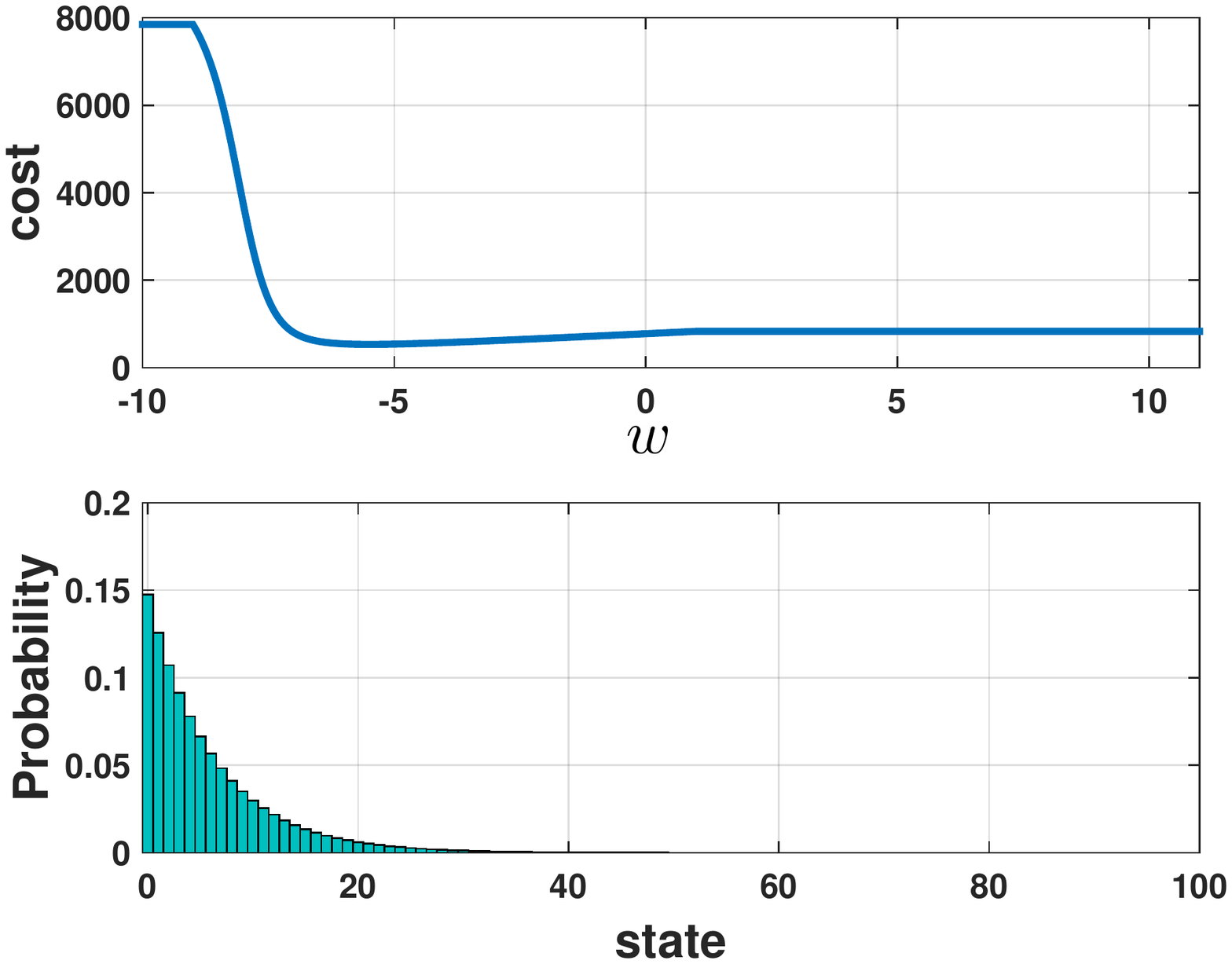} 	}} 
    \subfloat[]{{\includegraphics[trim = 10mm 50mm 6mm 50mm,height=4.5cm]{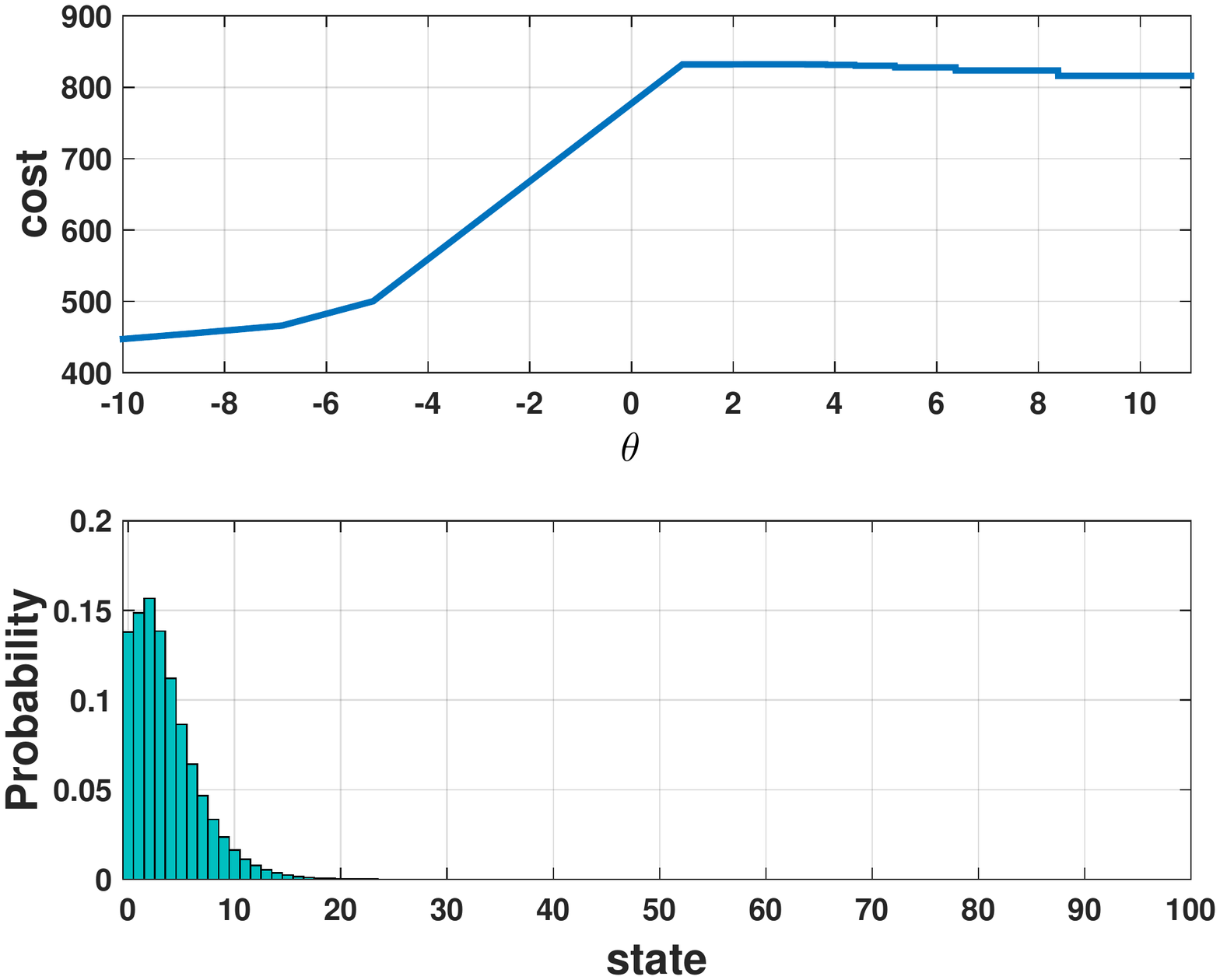} 	}} 
    \vspace{0cm}
    \caption{(a) Stationary distribution of the two initial policies. (b) Top: Cost of the mixture policy versus $w$. Bottom: Stationary distribution of the best mixture policy found in the space of $w$. (c) Top: Cost of the mixture policy versus $\theta$. Bottom:  Stationary distribution of the best mixture policy found in the space of $\theta$.} \label{fig:experiment1}%
\end{figure}
\subsection{Queuing Problem: Four Queues}
Consider a system with four queues, each has capacity $L=9$, shown in figure \ref{fig:fourqs}. This problem has been studied in several works \cite{chen1999value,kumar1990dynamic,de2003linear}. There are two servers in the system. Server $1$ serves queue $1$ and queue $4$ with rates $r_1$ and $r_4$, respectively, and server $2$ serves queue $2$ and queue $3$ with rates $r_2$ and $r_3$, respectively. Each server serves only one of its associated queues at each time. Jobs arrive at queue $1$ and queue $3$ with rate $\lambda$.  A job leaves the systems after being served at either queue $1$ and queue $2$ or queue $3$ and queue $4$. State of the system is denoted by a four dimensional vector $[x_1,x_2,x_3,x_4]$, where $x_i$ represents the number of jobs in queue $i$ at each time. A controller can choose a four-dimensional action from $\{0,1\}^4$ such that $a_1+a_4 \leq 1$ and $a_2 + a_3 \leq 1$. The cost at each time is equal to the total number of jobs in the system: $c(x) = \sum_{i=1}^4 x_i$.

\begin{figure}[!h]
    \centering
\includegraphics[trim = 0mm 0mm 0mm 0mm,width=8cm]{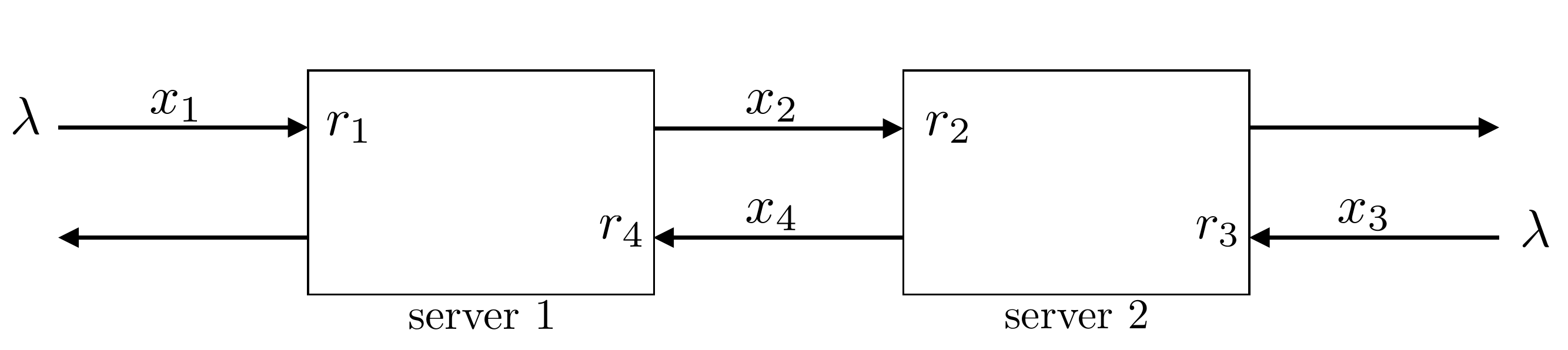} 
    \caption{A system with four queues and two servers} \label{fig:fourqs}%
\end{figure}

Assume $r_1 = r_2 = 0.12$, $r_3 = r_4 = 0.28$, and $\lambda = 0.08$. 
We choose our base policies from a family of policies for which server $i$ serves its longer queue with probability $p_i$ and its shorter queue with probability $1-p_i$ , for $i = \{1,2\}$. Five base policies and their associated costs are shown in Table \ref{tbl:base_policies}. 

Solving this problem in the space of policies using policy gradient will result in the optimal weight vector $w^* = [0 , 0 , 0 ,  1 , 0]$, i.e. giving all the weight to the policy with the lowest cost. Therefore the cost of the mixture policy will be $J(\pi^*) = 13.6525$. 

Interestingly, if we solve the problem in the dual space (space of stationary distributions) using policy gradient and Eq. \ref{eq:dual-policy}, after 200 iterations, the optimal resulting policy will have cost $J(\pi_{\theta}) = 12.84$ with $\theta = [-0.32, -0.68, -0.4, 2.16, 0.24]$ (note that based on the definition of $\Theta$, this vector can have negative values). However, this method is not computationally efficient. Using our stochastic subgradient method, we can approximate the cost very fast and efficient. In fact, the resulting policy will have cost $J(\pi_{\theta}) = 13.00$  with $\theta = [-0.23, -1.28,  0.09,   1.54, 0.88]$. Figure \ref{fig:theta4q} shows the difference between the costs of policy gradient and the stochastic subgradient method at each iteration. 

\begin{figure}[!h]
\centering
\includegraphics[width = 7cm]{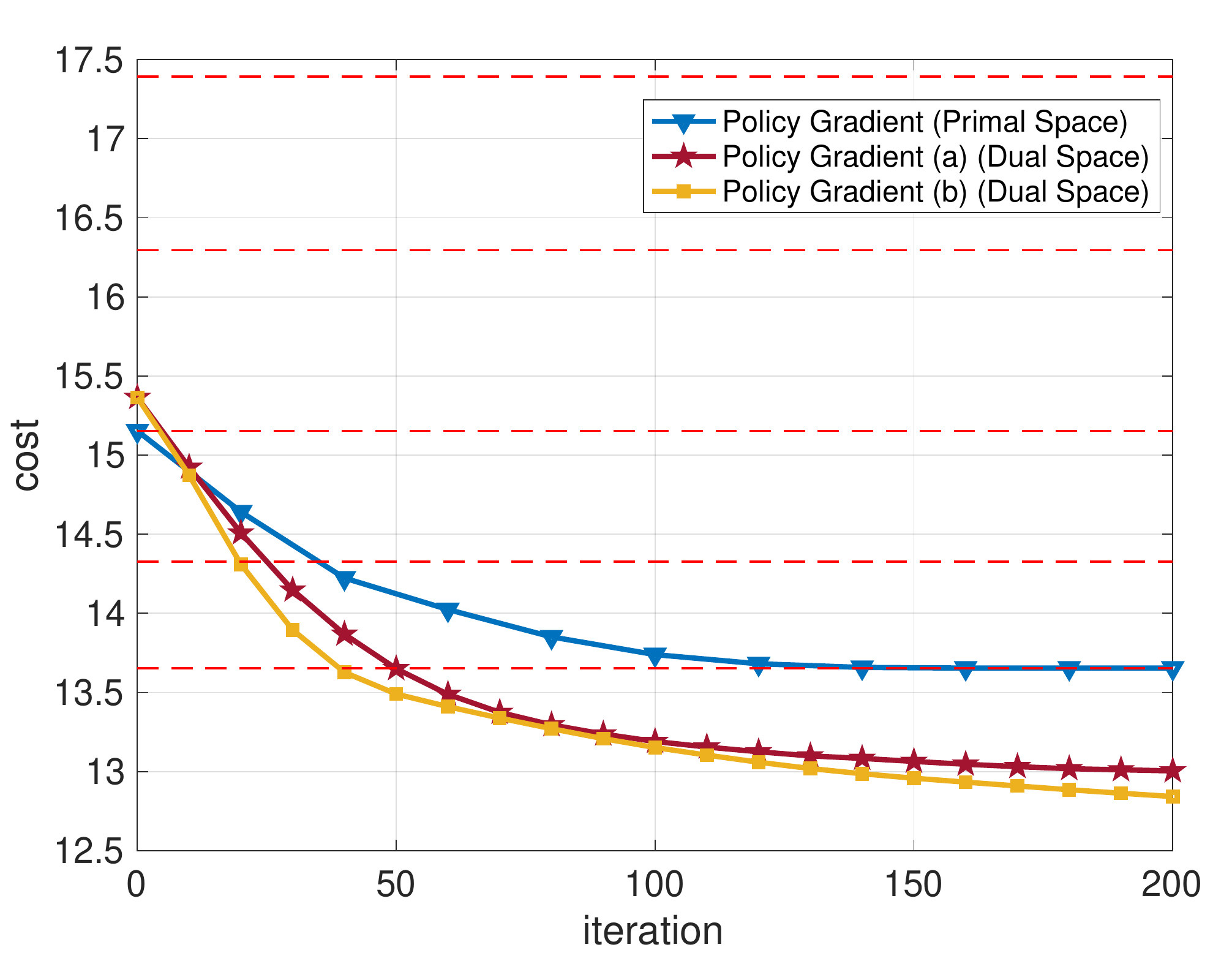}
\caption{cost per iteration (step) for the primal and the dual space. The policy gradient (b) for the dual space is the method described in algorithm \ref{alg:SGD}. Horizontal dashed lines are the costs of the base policies.}
\label{fig:theta4q}
\end{figure}

\begin{table}[!h]
\caption{Base Policies for four queue problem}
\vspace{-.25cm}
\small
\begin{center}
\begin{tabular}{c| c | c | c}
 & $p_1$ & $p_2$ & cost \\  \hline \hline
1	&0.9	&0.9	&16.2950  \\ 
 \hline
2	&0.9	&0.7	&17.3926 \\ 
 \hline
3	&0.8	&0.8	&15.1535 \\ 
 \hline
4	&0.7	&0.9	&13.6525 \\ 
 \hline
5	&0.7	&0.7	&14.3266 \\ 
\end{tabular}
\end{center}
\vspace{-.5cm}
\label{tbl:base_policies}
\end{table}

\begin{figure*}[!t]
\centering
\includegraphics[width = 11cm]{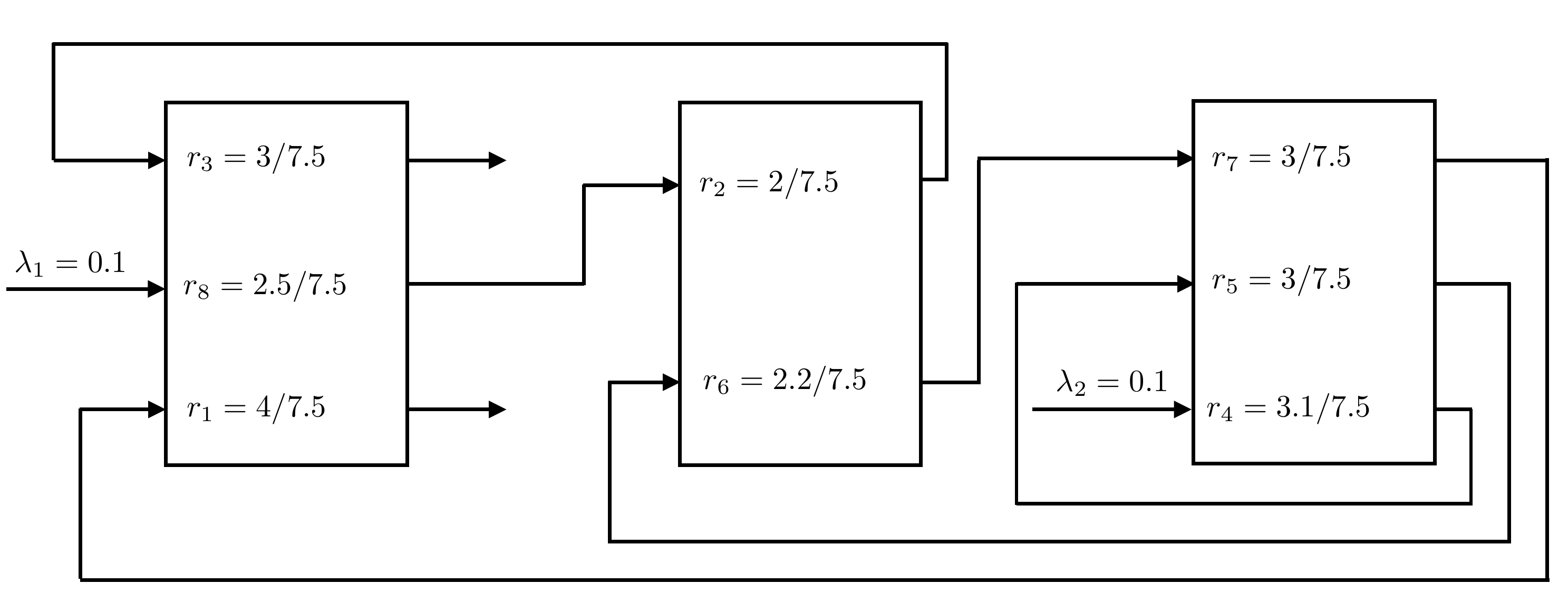}
\caption{8-queue Problem}
\label{fig:8q_sys}
\end{figure*}

\subsection{Queuing Problem: Eight Queues}
We consider a queuing problem with eight queues similar to the problem in \cite{de2003linear} with slightly different parameters. Consider a system with three servers and eight queues. Similar to  4-queue problem, each server  is responsible of serving multiple queues. There are two pipelines. Jobs in the first pipeline enter the system from the first queue by rate $\lambda_1$ and exit the system from the third queue. Jobs in the second pipeline enter the system from the forth queue by rate $\lambda_2$ and exit the system from the eighth queue. Fig. \ref{fig:8q_sys} demonstrates the system. The service rate is denoted by $r_i$, $i \in \{1,2,...,8\}$. There is no limit on the length of queues. State of the system is determined by an 8-dimensional vector that shows the number of jobs in each queue. Actions are also 8-dimensional binary vectors, where each dimension shows if the associated queue is served or not. The cost at each time step is equal to the total number of jobs in the system. Each server can serve only one queue at each time.  We choose three base policies from a family of policies that is described below. 

\textbf{Family of base policies:} The base policies are parameterized by three components: $p_1$, $p_2$, and $p_3$, which are associated to each server and satisfy the following properties:

\begin{itemize}
\item{If all of the queues associated to server $i$ are empty the server is idle. }
\item{If only one of the queues for server $i$ is non-empty, that queue is served with probability $p_i$.}
\item{If server $i$ has more than one non-empty queue, it serves the longest queue with probability $p_i$ and the rest of the non-empty queues with probability $(1-p_i)/(n_{q}-1)$, where $n_q$ is the  number of non-empty queues.}
\end{itemize} 

The three base policies and their costs are shown in Table \ref{tbl:base_policies_8}. The cost is computed by running each policy for 10,000,000 iterations, starting from empty system, and averaging over number of jobs at each iteration.

\begin{table}[!h]
\caption{Base Policies for eight queue problem}
\vspace{-.25cm}
\small
\begin{center}
\begin{tabular}{c| c | c | c | c }
 & $p_1$ & $p_2$ & $p_3$ & cost \\  \hline  \hline
1	&0.8		&0.5 	&0.5	& 21.78$\pm$ 0.16 \\ 
 \hline
2	&0.5		&0.8 	&0.5	&  22.47$\pm$0.08\\ 
 \hline
3	&0.5		&0.5	&0.8	& 22.98 $\pm$ 0.12\\
\end{tabular}
\end{center}
\vspace{-.5cm}
\label{tbl:base_policies_8}
\end{table}

Below we discuss about combining the policies in the primal and dual space.

\textbf{Solving the problem in the policy space:} Compared to the previous problems, the number of states in this domain is not limited. The cost of each policy should be computed by running the policy for a long time, e.g. 10,0000,000 iterations. We solve the problem in the primal space using policy gradient, and specifically a finite differences method. We start from a random initial wight. At each iteration, we slightly perturb the weight in different directions, compute the cost for each perturbation, and approximate the gradient surface. The weight is then updated in the direction of the gradient. Since computing the cost is computationally expensive, finding the optimum mixing weight in this space is very tedious.  After 50 iterations the weight converges to $w^*=[ 0.3519 , 0.3273 , 0.3209]$, which corresponds to $J(\pi^*)= 19.14$.

\textbf{Solving the problem in the space of occupancy measures:} To solve the problem in this space we only need to run the base policies once and obtain the stationary distributions. From there, we do not need to run any other policy and the optimal $\theta$ is obtained using the described algorithm. The optimal value of $\theta$ for this problem is: $\theta = [1.1246 , -0.1143 , -0.0102]$ and the cost induced is $J(\pi_{\theta})=  20.59$. In fact, by mixing the policies in this space we obtain a lower cost compared to the base policies. Although, this cost is slightly worse than the cost of the policy obtained the primal space, the gain in computation is significant.   

\begin{figure}[!h]
\centering
\includegraphics[width = 7cm]{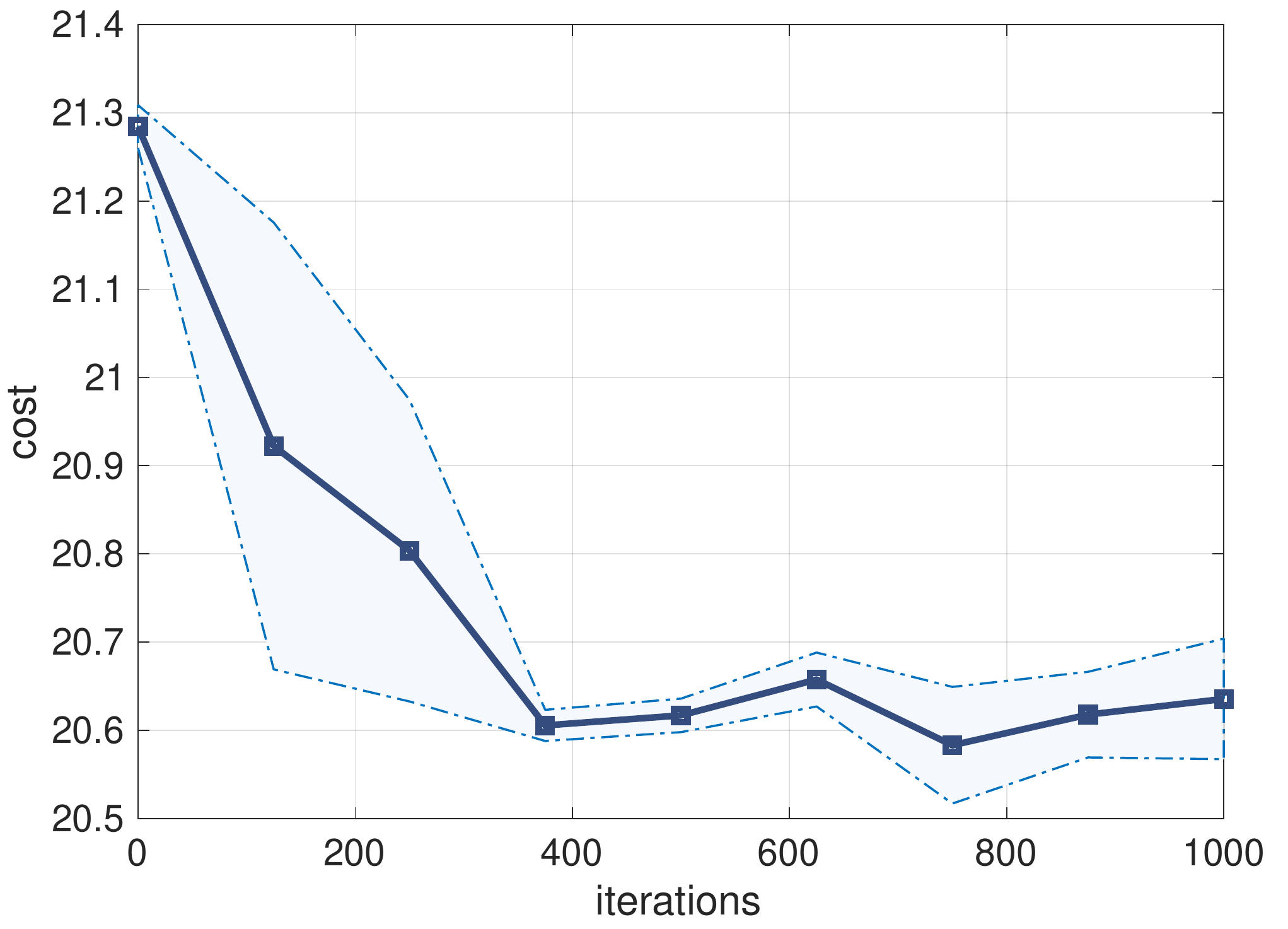}
\caption{Mean and standard deviation of the cost per iteration (step) in dual space in 8 queue problem.}
\label{fig:theta4q}
\end{figure}

\newpage
\bibliography{biblio}

\end{document}